%% file: main.tex
\documentclass{article}

\input{preamble/packages}

\input{preamble/macros}

\title{LBI: Parallel Scan Backpropagation via Latent Bounded Interfaces}

\author{%
  Shaun Christopher Lee \\
  University of California, Irvine \\
  \texttt{shauncl@uci.edu} \\
  \And
  Sangeetha Abdu Jyothi \\
  University of California, Irvine \\
  \texttt{sangeetha.aj@uci.edu} \\
}

\begin{document}

\maketitle

\input{sections/00-abstract}       
\input{sections/01-introduction}

\input{sections/03-mathematical-framework}

\input{sections/04-implementation}

\input{sections/05-experiments}
\input{sections/06-related-work}

\input{sections/07-discussion}

\input{sections/08-conclusion}

\bibliographystyle{plain}
\small
\bibliography{refs}
\normalsize

\appendix

\section{Vertex Separators}\label{app:vertex-separators}

\begin{definition}[Interface Separator]\label{def:interface-separator}
A set of nodes $S_k \subset \mathcal{V}$ is an \emph{interface separator} at boundary $k$ if every directed path from $\mathcal{R}_{<k}$ to $\mathcal{R}_{\geq k}$ contains at least one node in $S_k$.
\end{definition}

\begin{lemma}[A Separator Implies Sufficiency]\label{lem:separator-sufficiency}
If $S_k$ is an interface separator at boundary $k$, then the vector $m_k$ formed by collecting the values of nodes in $S_k$ is sufficient in the sense of Definition~\ref{def:sufficiency}.
\end{lemma}
\begin{proof}
We proceed by induction on the topological order of $\mathcal{G}$, restricted to $\mathcal{R}_{\geq k}$. Let $v \in \mathcal{R}_{\geq k}$ and suppose the claim made by the lemma holds for every node in $\mathcal{R}_{\geq k}$ that precedes $v$ in topological order. If $v \in S_k$, then the value of $v$ is a component of $m_k$ and the claim immediately holds. If $v \notin S_k$, then consider the parents of $v$ in $\mathcal{G}$. Notice that each parent $u$ of $v$ satisfies exactly one of the following conditions:
\begin{enumerate}
    \item $u \in \mathcal{R}_{\geq k}$. Then, $u$ precedes $v$ in topological order, so by the inductive hypothesis, $u$ is a function of $(m_k, \theta_{\geq k})$.
    \item $u \in \mathcal{R}_{<k}$. Then the edge $u \to v$ is a directed path from $\mathcal{R}_{<k}$ to $\mathcal{R}_{\geq k}$, which must contain a node in $S_k$ by Definition~\ref{def:interface-separator}. Since this path consists of only the two nodes $u$ and $v$, and $v \notin S_k$, we conclude that $u \in S_k$. Therefore the value of $u$ is a component of $m_k$.
\end{enumerate}
In either case, the value of every parent $u$ of $v$ is determined by $(m_k, \theta_{\geq k})$. Since $v$ computes a deterministic function of its parents, $v$ is also a function of $(m_k, \theta_{\geq k})$, which completes the induction.
\end{proof}

\begin{remark}[Minimum Separator Size]\label{rem:min-separator-size}
By Menger's vertex theorem~\citep{bohme2001menger}, the minimum vertex separator between $\mathcal{R}_{<k}$ and $\mathcal{R}_{\geq k}$ in a standard feedforward computational graph has size $d_k$, which is the full hidden state since each scalar component admits an independent directed path across the boundary. The bounded-interface property (Definition~\ref{def:bounded-interface-model}) therefore requires architectural modification. The encoder/decoder projections of Section~\ref{subsec:bi-architecture} create a new vertex separator of size $r_k \ll d_k$ in the modified graph, which cannot be achieved by selecting a subset of existing nodes.
\end{remark}

\paragraph{Graph-Theoretic vs. Effective Minimality.} Remark~\ref{rem:min-separator-size} establishes that exact sufficiency formally requires the full hidden state $d_k$. However, our experiments (Section~\ref{sec:experiments}) demonstrate that interfaces of dimension $r_k \ll d_k$ preserve training quality, indicating that inter-region communication concentrates in a low-dimensional subspace. The spectral analysis in Appendix~\ref{app:jacobian-structure} provides direct evidence. The interface Jacobians exhibit spectral concentration, with the leading singular value approximately $2.5\times$ the RMS singular value at $r=64$, suggesting that even within the $r$-dimensional interface, most gradient transport energy flows through a small number of directions. A second factor is the shared canvas $x_\mathrm{embed}$ (Section~\ref{subsec:bi-architecture}), which provides every region with the full $D$-dimensional token embedding independently of the interface chain. Because the canvas carries base token information, the interface need only encode inter-region refinements, reducing the effective information gap between $d_k$ and $r_k$ beyond what the learned function structure would predict. Characterizing the effective interface rank (the minimum $r_k$ for which the bounded-interface model achieves a given approximation quality) likely requires tools beyond the graph-theoretic framework developed here. Two natural directions are: (i) an information-theoretic analysis that bounds the mutual information $I(\mathcal{R}_{\geq k}; \mathcal{R}_{<k} \mid m_k)$ as a function of $r_k$, relating interface dimension to the information bottleneck principle~\citep{tishby2000information}, and (ii) a spectral analysis of the full-rank Jacobian $\partial h_k/\partial h_{k-1}$ at convergence, where the effective rank of this operator would predict the minimum interface dimension needed to preserve gradient transport quality. Both directions connect the architectural imposition (bounded interfaces) to the learned structure of the model and would provide principled guidance for choosing $r_k$.

\section{Representative Cost Instantiations}\label{app:representative-instances}

We derive the entries of Table~\ref{tab:transport-size-comparison}, then instantiate the work--span formulas of Proposition~\ref{prop:work-span-decomposition} for SSM-like and Transformer-like regions, assuming uniform interface width $r_k = r$. We analyze $r=64$ throughout. Since all per-combine costs scale as $O(r^3)$ and Jacobian construction scales as $O(r)$, the analysis is strictly more favorable at the $r=16$ that our experiments (Section~\ref{sec:experiments}) show suffices for model quality.

\begin{remark}[Derivation of Table~\ref{tab:transport-size-comparison}]\label{rem:table-transport-size-derivation}
Table~\ref{tab:transport-size-comparison} compares three regimes for transporting adjoints across $K$ region boundaries, where each boundary involves a $d \times d$ Jacobian ($d = BLD$) in the full-rank case or an $r \times r$ Jacobian in the bounded-interface case.

\emph{Sequential backpropagation.} Standard reverse-mode AD~\citep{baydin_AD} applies each $J_k^\top$ as a matrix-vector product (Theorem~\ref{thm:bi-backprop}): given $\bar{m}_{k+1} \in \mathbb{R}^d$, the product $J_k^\top \bar{m}_{k+1}$ costs $\Theta(d^2)$ FLOPs and reads $\Theta(d^2)$ data (the implicit Jacobian row/column entires accessed during the VJP), giving arithmetic intensity~\citep{williams2009roofline} $I = \Theta(1)$. The $K$ matvecs execute sequentially, yielding span $\Theta(Kd^2)$. The Jacobian is never materialized, but accessed implicitly through the AD tape.

\emph{Full-rank scan.} Replacing the sequential chain with a parallel scan (Proposition~\ref{prop:scan-parallel-backprop} with $r=d$) requires materializing each $J_k \in \mathbb{R}^{d \times d}$ and composing them via matrix-matrix multiplication. Each combine costs $\Theta(d^3)$ with arithmetic intensity $\Theta(d)$ (standard matmul on a $d \times d$ operand with $\Theta(d^2)$ data). The scan has $O(\log K)$ sequential combine steps, giving span $\Theta(d^3 \log K)$.

\emph{LBI scan.} Identical to the full-rank scan with $r$ replacing $d$. Then, the per-combine cost is $\Theta(r^3)$, the span is $\Theta(r^3 \log K)$, and the arithmetic intensity is $\Theta(r)$. The Jacobian $J_k \in \mathbb{R}^{r \times r}$ must be materialized, and its construction cost is analyzed below and excluded from Table~\ref{tab:transport-size-comparison}.
\end{remark}

\paragraph{Cost Structure.} Table~\ref{tab:representative-instances} summarizes the per-region costs for two canonical region types. For both, Jacobian construction multiplies forward FLOPs by $r$ while leaving the dominant memory traffic unchanged \emph{under activation reuse}. All $r$ perturbation directions share the same forward activations, so only lower-order tangent intermediates contribute additional traffic. The resulting arithmetic intensity is
\begin{equation}\label{eq:arithmetic-intensity-reuse}
    I(J_k) = rF_k/Q_k = r\cdot I_\mathrm{fwd},
\end{equation}
where $I_\mathrm{fwd} = F_k / Q_k$ is the arithmetic intensity of the forward pass alone.
\begin{table}[h!]
\centering
\begin{tabular}{c|c|c}
& SSM-like Region & Transformer-like Region \\
\hline
Activation shape & $B L D$, state $N$ & $B L D$, heads $H$, MLP $X$ \\
Forward FLOPs ($F_k$) & $\Theta(B L D N)$ & $\Theta(B L D^2 + B L^2 D + B L D X)$ \\
Forward memory ($Q_k$) & $\Theta(B L (D + N))$ & $\Theta(B L D + B H L^2 + B L X)$ \\
Jacobian FLOPs ($W_k^J$) & $\Theta(r \cdot B L D N)$ & $\Theta(r \cdot [B L D^2 + B L^2 D + B L D X])$ \\
Jacobian memory & $\Theta(B L (D + N))$ & $\Theta(B L D + B H L^2 + B L X)$ \\
Jacobian $I$ & $\Theta(r \cdot D N / (D + N))$ & $\Theta(r \cdot (D^2 + L D + D X) / (D + H L + X))$ 
\end{tabular}
\caption{Per-region cost of Jacobian construction for SSM-like and Transformer-like regions under uniform interface width $r$. Jacobian FLOPs scale by $r$ relative to the forward pass. Jacobian memory is unchanged under perfect activation reuse (Eq.~\eqref{eq:arithmetic-intensity-reuse}). $I$ denotes arithmetic intensity. Region subscripts are dropped for readability.}
\label{tab:representative-instances}
\end{table}

\paragraph{Activation Reuse in Practice.} Eq.~\eqref{eq:arithmetic-intensity-reuse} assumes that all $r$ perturbation directions share a single read of the forward activations. This holds when the $r$ directions are processed in a single fused kernel pass. In practice, memory constraints may require processing in chunks of size $c \leq r$ (Appendix~\ref{app:implement-details}), in which case the forward activations are read $\lceil r/c \rceil$ times and the effective arithmetic intensity becomes
\begin{equation}
    I_\mathrm{eff}(J_k) \approx c \cdot I_\mathrm{fwd},
\end{equation}
interpolating between $I_\mathrm{fwd}$ (no reuse, $c=1$) and $r\cdot I_\mathrm{fwd}$ (perfect reuse, $c=r$). We instantiate these formulas for a configuration representative of medium-scale language models ($\sim$1B parameters): $B=8$, $L=2048$, $D=768$, $N=16$ (SSM state size), $H=12$, $X=3072$ (Transformer MLP width), $r=64$, and $K=16$ (corresponding to 32 layers at 2 layers per region). These dimensions are chosen to illustrate the cost regimes at a scale where region-parallel training provides practical benefit. Empirical validation at this scale requires the native kernel implementation discussed in Section~\ref{sec:disc}. The compute-bound threshold on an H100-SXM5 in bf16 is approximately 295 ops/byte~\cite{nvidia-H100}. Table~\ref{tab:arithmetic-intensity-numbers} reports arithmetic intensity as a function of chunk size for both region types. 
\begin{table}[h!]
\centering
\begin{tabularx}{\linewidth}{X|c|c|c|c|c}
Region Type & $I_{\mathrm{fwd}}$ & $c=1$ & $c=16$ & $c=64$
  & H100 threshold \\
\hline
SSM ($B\!=\!8, L\!=\!2048, D\!=\!768, N\!=\!16$)
  & 7.8 & 7.8 & 125 & 500 & 295 \\
\hline
Transformer ($B\!=\!8, L\!=\!2048, D\!=\!768, H\!=\!12, X\!=\!3072$)
  & 80 & 80 & 1275 & 5100 & 295 \\
\end{tabularx}
\caption{Arithmetic intensity (ops/byte) of Jacobian construction as a function of chunk size $c$, for the illustrative configuration described above ($r=64$, bf16). Values above the H100 threshold indicate compute-bound operation. The SSM case requires $c \geq 38$ to cross the threshold, the Transformer case crosses it at $c \geq 4$.}
\label{tab:arithmetic-intensity-numbers}
\end{table}
The Transformer case is compute-bound even at modest chunk sizes because its forward pass is already moderately compute-dense ($I_\mathrm{fwd} \approx 80$). The SSM case is more demanding, where the forward pass is memory-bound ($I_\mathrm{fwd} \approx 7.8$), so reaching compute-bound operation requires chunks of $c \geq 38$. With a fused kernel that processes all $r=64$ directions simultaneously (the kernel engineering target discussed in Section~\ref{sec:disc}), both cases are firmly compute-bound.

\paragraph{Scan Negligibility.} The suffix scan over $K$ interface Jacobians of size $r \times r$ has total work $\Theta(Kr^3)$. For the illustrative configuration ($K=16, r=64$), the scan costs $Kr^3 = 16 \times 64^3 \approx 4.19 \times 10^6$ FLOPs. Total Jacobian construction for the SSM case costs $K\cdot rF_k = 16 \times 64 \times 2.01 \times 10^8 \approx 2.06 \times 10^{11}$ FLOPs, giving a scan-to-construction ratio below $10^{-4}$. The Transformer case yields a ratio below $10^{-7}$. In both cases the scan is negligible, and the dominant cost of bounded-interface backpropagation is Jacobian construction at $(r{+}1)\times$ the work of standard reverse-mode backpropagation. For comparison, a full-rank scan ($r = d = BLD \approx 1.26 \times 10^7$) would require $\Theta(d^3) \approx 2.0 \times 10^{21}$ FLOPs per combine, approximately $10^{16}\times$ the bounded-interface per-combine cost of $r^3 \approx 2.6 \times 10^5$ FLOPs.

\section{Execution \& Scheduling}\label{app:exec-and-schedule}

Algorithm~\ref{alg:bi-backprop} presents bounded-interface backpropagation as three strictly sequential phases: Jacobian construction, scan composition, and region-local backward. This ordering is correct but still conservative, and does not exploit the fact that Jacobian construction for region $k$ depends only on region $k$'s forward cache $C_k$ and the incoming interface state $m_k$, both of which are available as soon as region $k$'s forward pass completes. Algorithm~\ref{alg:bi-backprop-streaming} gives a streaming schedule that overlaps Jacobian construction with the sequential forward pass, reducing the effective wall-clock time between the end of the forward pass and the start of the scan.

\begin{algorithm}[h!]
\caption{Forward-Overlapped Bounded-Interface Backpropagation}
\label{alg:bi-backprop-streaming}
\begin{algorithmic}[1]

\Require Region transition maps $\{R_k\}_{k=0}^{K-1}$, local parameters $\{\theta_k\}_{k=0}^{K-1}$, input interface state $m_0$, terminal interface adjoint $\bar m_K$
\Ensure Interface adjoints $\{\bar m_k\}_{k=0}^{K}$ and parameter gradients $\{\nabla_{\theta_k} \mathcal{L}\}_{k=0}^{K-1}$

\Statex $\rhd$ Combined forward pass and Jacobian construction
\For{$k=0, \dotsc, K{-}1$}
    \State Compute $m_{k+1} = R_k(m_k; \theta_k)$ and store forward cache $C_k$
    \State \textbf{launch async:} Construct $J_k = \partial m_{k+1} / \partial m_k \in \mathbb{R}^{r \times r}$ from $C_k$ via AD
\EndFor
\State \textbf{synchronize} all Jacobian construction tasks

\Statex $\rhd$ Interface adjoint scan
\State Compute suffix products $\{P_k\}_{k=0}^{K}$ via parallel scan over $\{J_k^\top\}_{k=0}^{K-1}$
\ForAll{$k=0, \dotsc, K$ \textbf{in parallel}}
    \State $\bar{m}_k \leftarrow P_k \bar m_K$
\EndFor

\Statex $\rhd$ Region-local backward
\ForAll{$k=0, \dotsc, K{-}1$ \textbf{in parallel}}
    \State Compute $\nabla_{\theta_k} \mathcal{L} = (\partial m_{k+1} / \partial \theta_k)^\top \bar{m}_{k+1}$ using $C_k$
    \State Compute any additional local adjoints for region $k$
\EndFor

\Return $\{\bar m_k\}_{k=0}^K$ and $\{\nabla_{\theta_k} \mathcal{L}\}_{k=0}^{K-1}$
\end{algorithmic}
\end{algorithm}

The overlap window depends on the relative cost of the forward pass and Jacobian construction for each region. Since the forward pass costs $\Theta(F_k)$ FLOPs per region and Jacobian construction costs $\Theta(r_k \cdot F_k)$ (Eq.~\eqref{eq:bi-total-work}), Jacobian construction is approximately $r_k\times$ more expensive than a single region's forward pass. In the streaming schedule, construction for region $k$ proceeds concurrently with the forward passes of regions $k{+}1, \dotsc, K{-}1$. If the remaining $K{-}1{-}k$ forward passes provide sufficient time for region $k$'s Jacobian construction to complete, the synchronization barrier on line 5 incurs no additional wait. In the theoretical case where $K \gg r_k$, all Jacobians complete before the forward pass finishes, and the combined span of the forward pass plus Phase 1 collapses to approximately the span of the forward pass alone.

A second overlap opportunity exists between the scan and the region-local backward. Once the suffix scan produces $\bar m_k$ for a particular region $k$, that region's local backward can begin immediately without waiting for all interface adjoints. Exploiting this requires a scan implementation that exposes partial results incrementally, which is feasible with a streaming or segmented scan kernel but is not standard. We do not incorporate this optimization into Algorithm~\ref{alg:bi-backprop-streaming}.

Evaluating the wall-clock impact of these scheduling strategies requires architecture-specific kernel implementations and hardware profiling, which we defer to future systems work.

\section{Implementation \& Experimental Details}\label{app:implement-experiment-details}

\subsection{Implementation Details}\label{app:implement-details}

\paragraph{Jacobian Construction Procedure.} The core computational step unique to bounded-interface backpropagation is the construction of the region-local interface Jacobians $J_k := \partial m_{k+1}/\partial m_k \in \mathbb{R}^{r_{k+1} \times r_k}$. This Jacobian represents the local linearization of $R_k$: for any perturbation $\delta m_k \in \mathbb{R}^{r_k}$, the corresponding perturbation of the next interface state is $\delta m_{k+1} = J_k \, \delta m_k$. In general, $J_k$ is not available in closed form and must be constructed from the computation graph of $R_k$. We do so by propagating $r_k$ perturbation directions through the computation graph of $R_k$ via automatic differentiation (AD)~\citep{baydin_AD}. Let $\{e_j\}_{j=1}^{r_k}$ denote the standard basis of $\mathbb{R}^{r_k}$. A single AD pass through $R_k$ with tangent vector $e_j$ computes one column (or row) of $J_k$,
\begin{equation}\label{eq:basis-JVP}
    dR_k(m_k) \cdot e_j \in \mathbb{R}^{r_{k+1}},
\end{equation}
where $dR_k(m_k)$ denotes the differential (total derivative) of $R_k$ at $m_k$. This yields the $j$-th column (or row) of $J_k$. Repeating for all $r_k$ basis directions assembles the full Jacobian. When the underlying region kernel admits an additional batch dimension, all $r_k$ directions can be obtained simultaneously by propagating the $r_k \times r_k$ identity ($I_{r_k} \in \mathbb{R}^{r_k \times r_k}$) through the linearized computation
\begin{equation}\label{eq:jacobian-identity-construction}
    J_k = dR_k(m_k) \cdot I_{r_k},
\end{equation}
operating over tensors of shape $\mathbb{R}^{B_k \times L_k \times D_k \times r_k}$. This batched formulation amortizes kernel launch overhead and improves hardware utilization relative to $r_k$ sequential evaluations of~\eqref{eq:basis-JVP}. In our instantiation (Section~\ref{subsec:bi-architecture}), each evaluation propagates a perturbation through the composite path of Eq.~\eqref{eq:concrete-region-arch}. This step computes only the interface-to-interface Jacobian and does not require any parameter gradients, which are handled separately in the region-local backward phase (Section~\ref{subsec:backward-algorithm}).

\paragraph{Implementation and AD.} Our implementation provides two components for bounded-interface backpropagation: (i) a PyTorch-based~\citep{paszke2019pytorch} Jacobian materialization procedure that constructs each $J_k$ independently from cached boundary values through recomputation, and (ii) a dedicated CUDA suffix scan kernel that composes the $K$ interface Jacobians in a single kernel launch. Together, these enable correct end-to-end bounded-interface backpropagation as described in Algorithm~\ref{alg:bi-backprop}. Jacobian construction uses a recomputation strategy~\citep{chen2016training}: each region's cached boundary values ($m_k$ and $x_\mathrm{embed}$) are detached from the global computation graph, the batch dimension is replicated by a factor of up to $r$, the region forward map $R_k$ is re-executed on the expanded batch, and Jacobian rows are extracted via reverse-mode AD (VJPs) against the standard basis, equivalently yielding columns of $J_k^\top$. When memory is constrained, the basis dimension is processed in chunks of size $< r$, trading fewer simultaneous replicas for more forward passes. This approach is necessitated by PyTorch's execution model, which records the forward pass as a single sequential graph corresponding to the recurrence $m_0 \to R_0 \to m_1 \to \cdots \to m_K$, preventing region-local derivatives from being exposed as independent units of computation. The recomputation works around this by explicitly rebuilding independent per-region subgraphs from cached boundary values, achieving correctness but not native parallelism. Section~\ref{subsec:compute-cost} describes Jacobian construction in terms of the differential $dR_k(m_k)$, which is mode-agnostic. Both forward-mode AD (JVPs) and reverse-mode AD (VJPs) construct the same $r \times r$ Jacobian in $r$ passes at the same asymptotic cost (Eqs.~\eqref{eq:jacobian-work}--\eqref{eq:jacobian-arithmetic-intensity}). Our implementation uses reverse-mode with batch expansion, which interacts differently with the AD framework but produces identical results.

\paragraph{Systems Engineering.} Three systems capabilities named in Section~\ref{sec:disc} remain for future work and build on the algorithmic independence established in the main text. First, on fused Jacobian materialization, a dedicated kernel could fuse basis propagation with the underlying region computation in a single pass, achieving the $\Theta(r)$ arithmetic intensity described in Eq.~\eqref{eq:jacobian-arithmetic-intensity}, rather than incurring the overhead of batch replication and separate AD passes per chunk. Second, on overlapped scheduling, the dependency structure permits Jacobian construction for region $k$ to begin as soon as region $k$'s forward pass completes (Algorithm~\ref{alg:bi-backprop-streaming}, Appendix~\ref{app:exec-and-schedule}), and region-local backward for region $k$ to begin as soon as the scan produces $\bar m_k$. Exploiting this requires a streaming execution model that the current sequential graph recording does not support. Third, fused region-local backward passes that would exploit the inter-region independence established by Corollary~\ref{cor:intra-reg-local-bwd}, enabling all $K$ regions' parameter gradients to be computed concurrently.

\subsection{Experimental Details}\label{app:experiment-details}

\paragraph{Gradient Parity.} We verify that bounded-interface backpropagation (Algorithm~\ref{alg:bi-backprop}) recovers the same parameter gradients as standard reverse-mode AD applied to the same bounded-interface model. Because this test checks an implementation-level equality between two backward computations, rather than a scaling-dependent training outcome, we run it on reduced-size instances of each architecture chosen to exercise the same region decomposition, interface scan, and backend block types used in the experiments. For each architecture, we compare the full parameter-gradient vectors produced by the scan-based backward pass against the reference gradient from PyTorch's standard autograd, across 100 trials (20 random initializations $\times$ 5 batches). Table~\ref{tab:gradient-parity} reports worst-case values. The maximum absolute error, relative $\ell_2$ error, and cosine similarity confirm that the two gradient computations agree to the precision expected for each dtype. The larger discrepancies in the bfloat16 rows (Mamba-3 SISO, Hybrid) are consistent with reduced-precision arithmetic and differences in operation ordering from explicit Jacobian materialization, not from algorithmic error. All float32 architectures achieve cosine similarity $>0.99999$ with relative error below $10^{-7}$.

\begin{table}[h!]
\centering
\begin{tabular}{l|c|c|c|c}
Architecture (LBI) & $\|\Delta g\|_\infty$ & $\|\Delta g\|_2 / \|g\|_2$ & $\cos \mathrm{sim}$ & dtype \\
\hline
Mamba-2 & $1.12\times 10^{-8}$ & $2.07\times 10^{-8}$ & $>0.99999$ & float32 \\
Mamba-3 SISO & $4.88\times 10^{-4}$ & $1.66 \times 10^{-3}$ & $>0.99999$ & bfloat16 \\
Transformer & $8.94\times 10^{-8}$ & $1.41\times 10^{-7}$ & $>0.99999$ & float32 \\
Hybrid & $5.86 \times 10^{-3}$ & $1.37 \times 10^{-2}$ & $0.99991$ & bfloat16 \\
\end{tabular}
\caption{
    Gradient parity between bounded-interface backpropagation (Algorithm~\ref{alg:bi-backprop}) and standard reverse-mode AD, where all models are bounded-interface (LBI) variants of the corresponding architectures. Results are worst-case values over 100 trials per architecture (20 random initializations, 5 batches per initialization). We report the maximum absolute gradient error $\|\Delta g\|_\infty$, relative $\ell_2$ error $\|\Delta g\|_2 / \|g\|_2$, and cosine similarity between full parameter-gradient vectors.
}
\label{tab:gradient-parity}
\end{table}

\paragraph{Hyperparameter Tuning.} Bounded-interface models introduce a modified optimization problem relative to their dense counterparts: the encoder/decoder projections add new parameter groups, the interface bottleneck reshapes the loss surface, and the effective gradient scale at interface boundaries depends on the conditioning of the encoder/decoder Jacobians. We therefore treat the optimizer settings as a stability-oriented configuration rather than an exhaustive optimum. For dense baselines, we ran fixed-seed sweeps over learning rate, weight decay, and warmup, while holding the optimizer, learning-rate schedule, batch size, sequence length, tokenizer, model size, and gradient clipping fixed. We then fixed the selected architecture-specific dense configurations for all reported seeds. For LBI variants, we reuse the corresponding dense non-learning-rate optimizer settings, and perform further learning rate diagnostics where transfer from the dense configuration was unstable. Mamba-2 and Transformer transferred without modification to the LBI setting. Mamba-3 SISO required lowering the LBI learning rate from $6 \times 10^{-4}$ to $3 \times 10^{-4}$ for stable training. Hybrid used the dense learning rate ($6 \times 10^{-4}$) for $r = 16$, but a lower learning rate ($3 \times 10^{-4}$) for $r \in \{32, 64\}$. These results suggest that LBI preserves the dense training configuration to first order, but its interface-induced gradient pathway can require architecture/rank-specific learning rate adjustment. A systematic optimizer study is a natural direction for future larger-scale experiments.

\paragraph{On Mamba-3 MIMO.} We evaluate Mamba-3~\citep{mamba-3} in its SISO configuration only. The MIMO configuration augments SISO with an additional rank dimension $R > 1$ that expands the B/C projections, and introduces per-head MIMO mixing projections for the value/gating/output pathways. This increases intra-region expressivity and kernel work while preserving the same external hidden-state dimension. The bounded-interface constraint applies only at region boundaries, and intra-region computation is otherwise unrestricted, so MIMO is compatible with the LBI formulation in principle. Support for the MIMO path requires additional systems/kernel engineering and is a direction for future work.

\paragraph{Hardware/Compute.} All reported training runs used single-GPU jobs on NVIDIA H100 NVL GPUs~\citep{nvidia-H100} with 96 GB memory. Reported runs (Figures~\ref{fig:train-val-curves},~\ref{fig:region-size-sweep}, and~\ref{fig:jacobian-spectral-norm-training}) trained for 20,000 steps with batch size 1 and sequence length 1024, corresponding to 20.48M tokens per run. On this hardware, dense baselines required approximately 10-17 minutes per run, while LBI runs required approximately 55--94 minutes for $r=16$, 71--139 minutes for $r=32$, and 105--230 minutes for $r=64$, depending on architecture. The 48 total runs required approximately 71 H100-hours. Appendix region-size sweeps added approximately 41 H100-hours, Jacobian analysis runs added approximately 12 H100-hours, for roughly 124 logged H100-hours across the listed experimental workflow. These estimates exclude startup, checkpoint I/O, tuning, dataset export/tokenization, and other exploratory runs.

\section{Region Size Sensitivity}\label{app:region-size}

We sweep the number of layers per region at fixed backend parameters, training budget, and interface rank $r=32$ for Mamba-3 SISO and Transformer backends. Varying the region size changes both $K$ (the number of interface states) and the encoder/decoder parameter overhead per region. All other hyperparameters are held fixed from the main experiments (Section~\ref{sec:experiments}).

\begin{figure}[h!]
\centering
\resizebox{1.0\linewidth}{!}{%
\begin{tabular}{cc}
\begin{subfigure}{0.5\linewidth}
    \includegraphics[width=\linewidth]{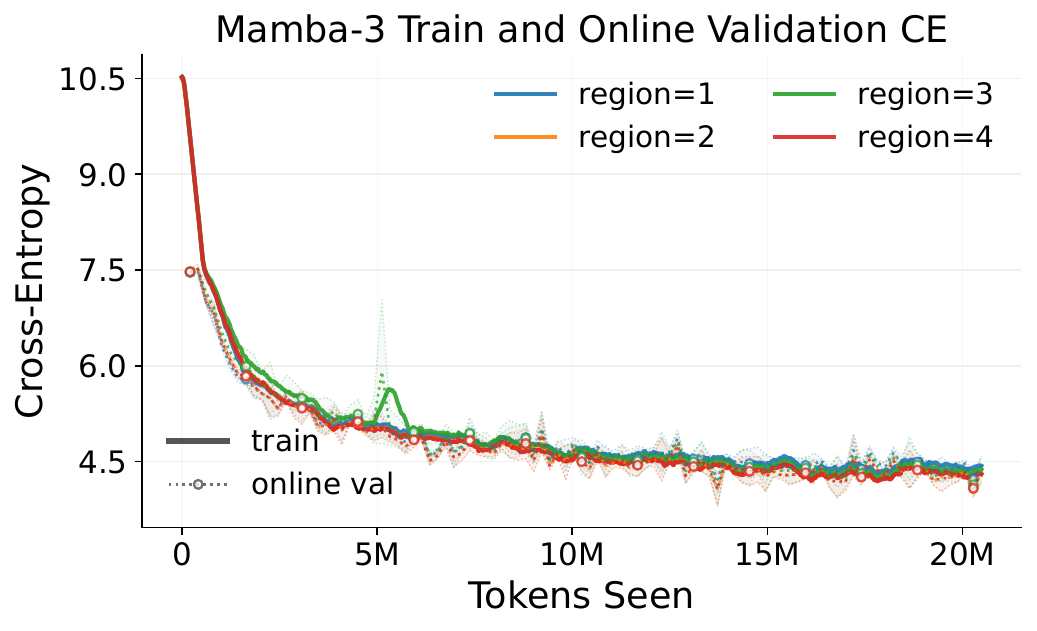}
    \caption{Mamba-3 SISO}
\end{subfigure} &
\begin{subfigure}{0.5\linewidth}
    \includegraphics[width=\linewidth]{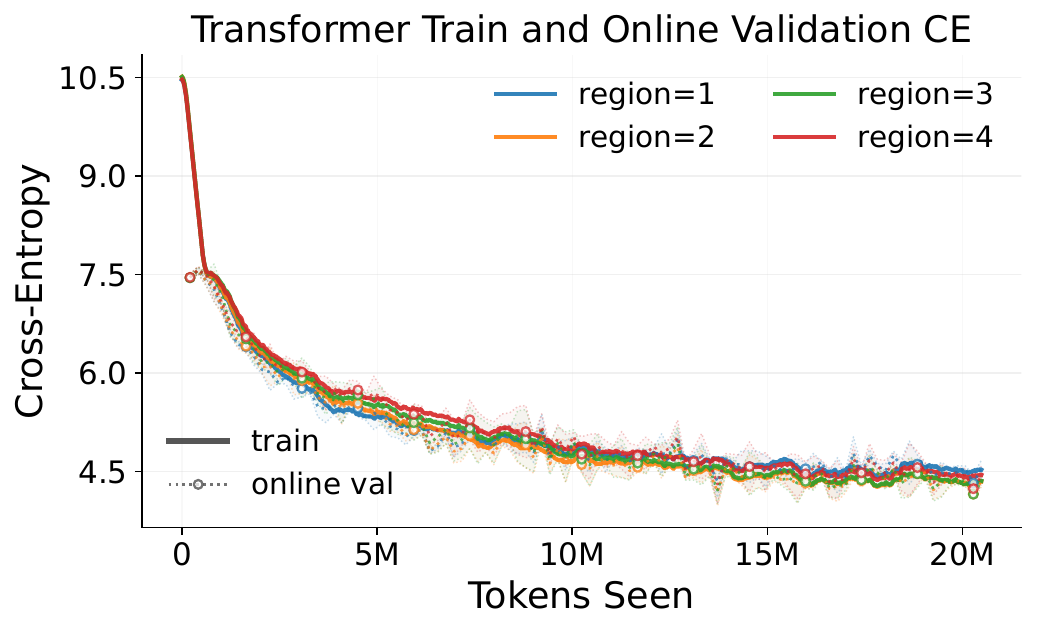}
    \caption{Transformer}
\end{subfigure}
\end{tabular}
}
\caption{Smoothed training CE (solid lines) and periodic online validation CE (markers) on FineWeb-Edu for Mamba-3 SISO (left) and Transformer (right) at interface rank $r=32$, varying the number of layers per region from 1 to 4. Training curves are smoothed with an exponential moving average for readability. Shaded regions indicate $\pm1$ standard deviation across three seeds for validation loss. Final quantitative comparisons are reported in Table~\ref{tab:region-size-val-loss-final}.}
\label{fig:region-size-sweep}
\end{figure}

\begin{table}[h!]
\centering
\small
\setlength{\tabcolsep}{4pt}
\resizebox{1.0\linewidth}{!}{%
\begin{tabular}{l|c|c|c|c|c|c}
Architecture (LBI) & Region Size & $K$ & Backend $N$ & Interface $N$ & Total $N$ & Post-Hoc Val CE $\downarrow$ \\
\hline
\multirow{4}{*}{Mamba-3 SISO}
& 1 & 14 & \multirow{4}{*}{53.52M} & 17.85M & 95.95M & $4.442 \pm 0.007$ \\
& 2 & 7 & & 9.23M & 87.33M & $4.310 \pm 0.004$ \\
& 3 & 5 & & 6.77M & 84.86M & $4.369 \pm 0.104$ \\
& 4 & 4 & & 5.54M & 83.63M
 & $4.309 \pm 0.004$ \\
\hline
\multirow{4}{*}{Transformer}
& 1 & 12 & \multirow{4}{*}{47.25M} & 6.98M & 70.62M & $4.533 \pm 0.027$ \\
& 2 & 6 & & 3.63M & 67.26M & $4.348 \pm 0.069$ \\
& 3 & 4 & & 2.51M & 66.15M & $4.349 \pm 0.123$ \\
& 4 & 3 & & 1.96M & 65.59M & $4.445 \pm 0.144$ \\
\end{tabular}
}
\caption{
    Post-hoc validation CE across region sizes for Mamba-3 SISO and Transformer at interface rank $r=32$ (same evaluation procedure as Table~\ref{tab:val-loss-final}). Region size denotes the number of consecutive layers per region, yielding $K = \lceil\#\text{ layers}/\text{region size}\rceil$ regions. When the division is uneven, the final region is truncated in the number of layers. Backend $N$ is fixed within each architecture, and the interface $N$ varies with region size due to the per-region encoder/decoder overhead. Entries report mean $\pm$ standard deviation across three seeds.
}
\label{tab:region-size-val-loss-final}
\end{table}

Region sizes of 2--4 layers yield comparable validation CE for both architectures. Region size 1 (one layer per region) incurs the largest CE for both backends (4.442 for Mamba-3, 4.533 for Transformer), which we attribute to two compounding factors: the maximum number of interface bottlenecks ($K=14$ and $K=12$ respectively), each forcing information through an $r$-dimensional compression, and the largest encoder/decoder parameter overhead (17.85M and 6.98M), which increases total parameter count without proportionally benefiting model quality. At region size 1, the encoder/decoder parameter count per interface approaches $\sim$33\% of each layer's parameter count for Mamba-3 ($\sim$1.28M vs. 3.82M), making the encoder/decoder a substantial nonlinear transformation that competes with the backend during optimization. For the Transformer, this ratio is lower ($\sim$15\%, $\sim$0.58M vs. $\sim$3.94M), indicating that parameter competition is less severe. Despite this, the Transformer also incurs its largest CE at region size 1, suggesting that the bottleneck frequency ($K=12$ compressions through an $r$-dimensional interface) is a more dominant degradation mechanism. The parameter competition may compound this effect for Mamba-3 but is secondary for the Transformer.

The overall pattern (Figure~\ref{fig:region-size-sweep} and Table~\ref{tab:region-size-val-loss-final}) shows a tradeoff between two possible failure modes. Too many interfaces (region size 1) impose excessive compression, and may distort the optimization landscape through heavy per-layer encoder/decoder overhead. Too few interfaces may provide insufficient inter-region communication points, increasing sensitivity to initialization and reducing the model's ability to propagate refinements across the network, but this degradation appears only for the Transformer at region size 4, not for Mamba-3. Importantly, increasing region size does not interpolate toward the unrestricted dense model, since the bounded-interface architecture rebuilds each region's hidden state from $x_\mathrm{embed} + \operatorname{Dec}_k(m_k)$ regardless of region size, so inter-region communication remains $r$-dimensional even when individual regions are large. The sweet spot of 2--3 layers per region in our experiments balances compression frequency against communication opportunity, which is a practical property since it means the framework does not require careful tuning of this hyperparameter.

\section{Interface Jacobian Structure}\label{app:jacobian-structure}

We report the spectral properties of the interface Jacobians $\{J_k\}$ throughout training for Mamba-3 SISO LBI at ranks $r \in \{16, 32, 64\}$, using the main experimental configuration (region size 2, $K=7$ regions). Figure~\ref{fig:jacobian-spectral-norm-training} tracks the local and suffix-composed spectral norms across training steps. Table~\ref{tab:jacobian-norm-final-checkpoint} reports summary statistics at the final checkpoint.

\begin{figure}[h!]
\centering
\includegraphics[width=\linewidth]{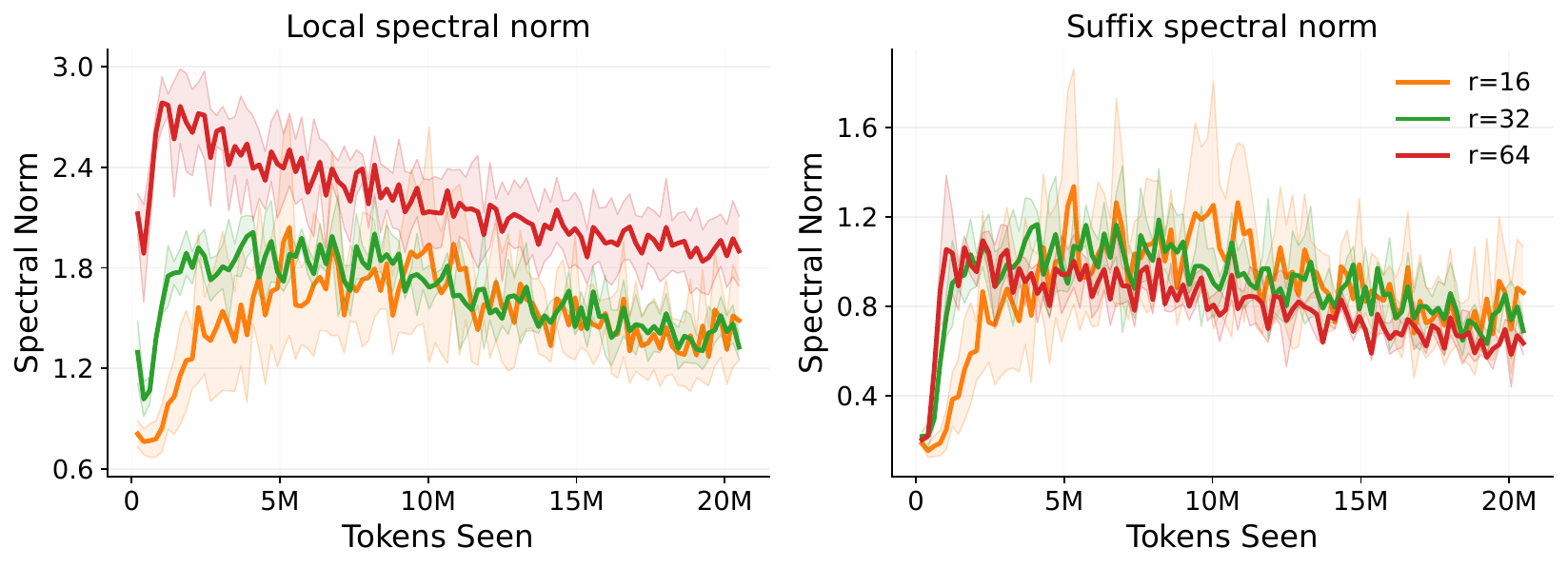}
\caption{Spectral norm of interface Jacobians throughout Mamba-3 SISO LBI training at ranks $r \in \{16, 32, 64\}$ with region size 2. Left: mean local spectral norm $\|J_k\|_2$ averaged across regions. Right: mean spectral norm of the suffix-composed interface Jacobian $\|P_k\|_2 = \|\prod_{j=k}^{K-1} J_j^\top\|_2$, averaged across suffix indices. Both quantities are averaged over three seeds, with shaded regions indicating $\pm 1$ standard deviation. The first logged step is omitted to remove the initialization transient.}
\label{fig:jacobian-spectral-norm-training}
\end{figure}

\begin{table}[h!]
\centering
\begin{tabular}{c|c|c|c}
Rank $r$ & Local $\|J_k\|_2$ & Suffix $\|P_k\|_2$ & Local $\|J_k\|_F/\sqrt{r}$ \\
\hline
16 & 1.486 $\pm$ 0.293 & 0.862 $\pm$ 0.252 & 0.708 $\pm$ 0.190 \\
32 & 1.327 $\pm$ 0.084 & 0.690 $\pm$ 0.071 & 0.591 $\pm$ 0.036 \\
64 & 1.901 $\pm$ 0.257 & 0.635 $\pm$ 0.060 & 0.749 $\pm$ 0.103 \\
\end{tabular}
\caption{Final-checkpoint interface Jacobian statistics for Mamba-3 SISO LBI at ranks $r \in \{16, 32, 64\}$ with region size 2. Local $\|J_k\|_2$ is the spectral norm of individual region-to-region interface Jacobians averaged across regions, suffix $\|P_k\|_2$ is the spectral norm of the suffix-composed interface Jacobians averaged across suffix indices, and local $\|J_k\|_F/\sqrt{r}$ is the rank-normalized Frobenius norm (equal to the RMS singular value when $J_k$ is square). All values are computed from the final logged training step and reported as mean $\pm$ standard deviation across three seeds. The initialization transient is again omitted.}
\label{tab:jacobian-norm-final-checkpoint}
\end{table}

\paragraph{Spectral Analysis.} We make three observations. First, individual region transitions have local spectral norms slightly above 1 (1.3--1.9 across ranks), indicating moderate amplification along the leading singular direction. However, the suffix-composed Jacobians (which govern the actual gradient transport through the interface chain) converge to spectral norms consistently below 1 (Figure~\ref{fig:jacobian-spectral-norm-training}), meaning that the composed interface map is contractive. Gradients propagated through the full interface chain are attenuated rather than amplified, which is consistent with the stable training observed in Section~\ref{sec:experiments}. The contraction of the composed map despite local amplification implies that the leading singular directions of adjacent Jacobians are not aligned, where amplification in one region is not reinforced by the next and the composed product is dominated by the contractive bulk of the spectrum rather than the leading singular values. Second, both local and suffix spectral norms stabilize rapidly after an initial transient and remain bounded throughout training for all three ranks. This suggests that the gradient transport structure of the interface chain is determined early on in training and maintained thereafter, rather than drifting or deteriorating. Third, the gap between the local spectral norm $\|J_k\|_2$ and the rank-normalized Frobenius norm $\|J_k\|_F/\sqrt{r}$ reveals spectral concentration within the interface Jacobians. For $r=64$, the spectral norm (1.901) is approximately 2.5$\times$ the RMS singular value (0.749), indicating that a few singular directions carry disproportionate energy while most directions are attenuative. This concentration is consistent with the observation from Section~\ref{sec:experiments} that small $r$ is sufficient, where even within the $r$-dimensional interface, the effective communication concentrates in a low-dimensional subspace. The canvas mechanism (Section~\ref{subsec:bi-architecture}) may contribute to this concentration, since the embedding already carries $D$-dimensional token information to every region, the interface communicates refinements that plausibly occupy a lower-dimensional subspace than the full hidden state would require.

\paragraph{Further Characterization.} We report results only for Mamba-3 SISO, so the Jacobian structure may differ for Mamba-2, Transformer, or Hybrid backends where the region transformation $\Phi_k$ has different spectral properties. A full per-region breakdown of $J_k$ would reveal whether early or late regions are systematically more or less contractive. The full singular value spectrum of each $J_k$ (beyond the leading value and the RMS) would characterize the effective rank of inter-region communication more precisely. Finally, the layer normalization in Eq.~\eqref{eq:concrete-region-arch} likely plays a stabilizing role in keeping the suffix norms bounded, since the normalizing effect on the pre-interface state constrains the output scale of each region transition. Characterizing the interaction between the LN, the learned scaling $\alpha_k$, and the encoder/decoder architecture in determining Jacobian conditioning is a direction for future work that connects to the open design problem discussed in Section~\ref{sec:disc}.

\end{document}

%% file: preamble/packages.tex
\usepackage{amsmath}
\usepackage{amssymb}
\usepackage{mathtools}
\usepackage{amsthm}

\usepackage{algorithm}
\usepackage{algpseudocode}

\usepackage[textsize=tiny]{todonotes}

\usepackage{enumitem}

 \usepackage[preprint]{neurips_2026}

\usepackage[utf8]{inputenc} %
\usepackage[T1]{fontenc}    %
\usepackage{hyperref}       %
\usepackage{url}            %
\usepackage{booktabs}       %
\usepackage{amsfonts}       %
\usepackage{nicefrac}       %
\usepackage{microtype}      %
\usepackage{xcolor}         %
\usepackage{multirow}
\usepackage{subcaption}
\usepackage{tabularx}

%% file: preamble/macros.tex
\theoremstyle{plain}
\newtheorem{theorem}{Theorem}[section]
\newtheorem{proposition}[theorem]{Proposition}
\newtheorem{lemma}[theorem]{Lemma}
\newtheorem{corollary}[theorem]{Corollary}
\theoremstyle{definition}
\newtheorem{definition}[theorem]{Definition}

\theoremstyle{remark}
\newtheorem{remark}[theorem]{Remark}

%% file: sections/00-abstract.tex
\begin{abstract}

Backpropagation is inherently sequential across depth, creating an $O(K)$-deep dependency chain that bottlenecks parallel training. While parallel-scan formulations theoretically reduce this depth to $O(\log K)$, they are computationally prohibitive for modern architectures due to the $O(d^3)$ cost of composing full-rank $d\times d$ Jacobians over the entire hidden state. We introduce Latent Bounded Interfaces (LBI), an algorithmic formulation that makes scan-based backpropagation tractable by restricting inter-region communication to a low-dimensional latent interface, $ m_k \in \mathbb{R}^{r}$, where $r \ll d$. This reduces the adjoint recursion to a suffix scan over $r \times r$ Jacobians, cutting per-combine cost from $O(d^3)$ to $O(r^3)$ while preserving exact gradients under the bounded-interface model. We demonstrate that LBI maintains model quality across four architectures (Mamba-2, Mamba-3, Transformer, and a Mamba--Transformer hybrid) at 47--61M block parameters. Interfaces of dimension $r=16$ suffice to preserve training quality within 0.16--0.35 cross entropy of dense baselines. The resulting framework provides an algorithmic foundation for region-parallel training, reducing cross-device backward communication to a single scan over $K$ fixed-size matrices, of approximately 56 KB for our experimental configurations.

\end{abstract}

%% file: sections/01-introduction.tex
\section{Introduction}\label{sec:intro}

Backpropagation in deep neural networks is inherently sequential across layers: the adjoint at layer $k$ depends on the adjoint at layer $k{+}1$, which depends on layer $k{+}2$ and so on. This sequential dependency chain has depth $O(K)$ in the number of regions (or layers) $K$, imposing a fundamental algorithmic bottleneck on parallel training. Prior efforts to break this sequential chain have taken broadly complementary approaches. Exact scan-based approaches, such as BPPSA~\citep{exact-para-scan-bp}, theoretically reduce the $O(K)$ dependency chain to $O(\log K)$, but are computationally prohibitive due to the $O(d^3)$ cost of composing $d \times d$ Jacobians (where $d = BLD$ encompasses the full batch, sequence, and feature dimensions). Iterative~\citep{iterative-residual,lim2024parallelizingnonlinearsequentialmodels,scale-parallel-nRNN,layer-para-rnn}, decoupled~\citep{de-info-reg,decoupled-neural-interfaces}, and Jacobian-Free~\citep{jfb} approaches sacrifice gradient exactness for throughput. Pipeline parallelism~\citep{interlocking-backprop,chimera,GPipe} uses micro-batch scheduling to improve hardware utilization but leaves the $O(K)$ sequential depth intact. 

We observe that the intractability of scan-based backpropagation is not inherent to the scan formulation itself, but to the size of the objects being scanned. Restricting inter-region communication to a low-dimensional \emph{interface state} $m_k \in \mathbb{R}^r$ reduces the adjoint recursion to a suffix scan over $r \times r$ Jacobians, preserving the $O(\log K)$-depth associative structure identified by BPPSA while cutting per-combine cost from $O(d^3)$ to $O(r^3)$ with $r \ll d$. All remaining gradient computation (parameter gradients and intra-region adjoints) becomes embarrassingly parallel across regions once the interface adjoints are available.

This approach instantiates a broader principle: that inter-component communication in neural networks admits learned low-dimensional compression. Multi-Head Latent Attention (MLA)~\citep{deepseek-v2} demonstrated this for cross-token attention, where attention can be compressed to a low-dimensional latent state without significant quality loss provided that the compression is native to the model and learned end-to-end. We apply the same principle to inter-region communication in depth. Rather than transmitting the full hidden state across region boundaries, we learn a compact interface carrying only the information needed by downstream computation. The resulting architecture, which we call a \emph{bounded-interface model}, represents an architecture--algorithm co-design: it is both a modeling choice (that imposes a structural information bottleneck) and an algorithmic enabler (that reduces scan operator size from $d \times d$ to $r \times r$). We make the following contributions:

\begin{enumerate}[leftmargin=*]
    \item \textbf{Tractable Scan Backpropagation via Bounded Interfaces.} We restrict inter-region communication to an $r$-dimensional interface and prove that the resulting adjoint recursion admits an exact suffix scan over $r\times r$ Jacobians, reducing per-combine cost from $O(d^3)$ to $O(r^3)$ with $O(\log K)$ depth. Our algorithmic foundation decouples the scan's computational cost from the model's internal dimensions, making depth-parallel backpropagation tractable for modern architectures.
    \item \textbf{Three-Phase Decomposition with Region Independence.} We show that bounded-interface backpropagation decomposes into three phases: embarrassingly parallel Jacobian construction, a lightweight suffix scan, and embarrassingly parallel region-local backward, where the scan is the sole sequential bottleneck at $O(r^3\log K)$ span. Once interface adjoints are available, all region-local computations are independent, enabling flexible distributed scheduling.
    \item \textbf{Exactness without Approximation.} We verify that the resulting gradients are identical to standard reverse-mode backpropagation under the bounded-interface model. Unlike prior approaches, the scan-based backward pass does not introduce approximation, staleness, or auxiliary objectives.
    \item \textbf{Model-Agnostic Realization and Empirical Validation.} We instantiate bounded-interface models across four architectures (Mamba-2, Mamba-3, Transformer, Hybrid) via a simple encoder/decoder construction that compresses inter-region communication to a fixed-dimensional interface, while leaving intra-region computation unrestricted. We demonstrate that interfaces of dimension $r=16$ suffice to preserve training quality within 0.16--0.35 cross entropy of dense baselines, establishing the practical viability of LBI. The code is available at \url{https://github.com/shaunlee8/latent-bounded-interfaces}.
\end{enumerate}

%% file: sections/03-mathematical-framework.tex
\section{Scan Formulation}\label{sec:math}

We formalize the notion of a bounded interface as a low-dimensional vertex separator in the computational graph through which all inter-region dependencies must pass, and show that this structure implies a factorization of both forward computation and reverse-mode adjoints.

\subsection{Graph and Interface Factorization}

\begin{definition}[Region-Partitioned Computational Graph]\label{def:region-partition}
Let $\mathcal{G} = (\mathcal{V}, \mathcal{E})$ denote the directed acyclic computational graph of a model, where $\mathcal{V}$ are the graph nodes and $\mathcal{E}$ are directed edges representing data dependencies. A \emph{region partition} of $\mathcal{G}$ is a disjoint decomposition,
$\mathcal{V} = \bigsqcup_{k=0}^{K-1} \mathcal{R}_k$
that is \emph{topologically ordered}: for every edge $(u \to v) \in \mathcal{E}$ with $u \in \mathcal{R}_k$ and $v \in \mathcal{R}_j$, we have $k \leq j$. This ensures that inter-region dependencies flow in a single forward direction and permits arbitrary region boundaries consistent with the DAG, including partitions that do not align with architectural layers. For each boundary between regions $\mathcal{R}_{k-1}$ and $\mathcal{R}_k$, we associate an \emph{interface state,} $m_k \in \mathbb{R}^{r_k}$, where $r_k$ denotes the interface dimension, which controls the amount of information communicated across region boundaries.
\end{definition}

\begin{definition}[Interface Sufficiency]\label{def:sufficiency}
An interface state $m_k$ is \emph{sufficient} at boundary $k$ if for every node $v \in \mathcal{R}_{\geq k}$, there exists a function $\tilde{f}_v$ such that $v = \tilde{f}_v(m_k; \theta_{\geq k})$. That is, all dependence of $v$ on variables in $\mathcal{R}_{<k}$ factors exclusively through $m_k$.
\end{definition}

Sufficiency holds whenever $m_k$ forms a \emph{vertex separator} between $\mathcal{R}_{<k}$ and $\mathcal{R}_{\geq k}$ in the computational graph, that is, every directed path from $\mathcal{R}_{<k}$ to $\mathcal{R}_{\geq k}$ passes through nodes whose values determine $m_k$. A formal proof that a vertex separator implies sufficiency (Definition~\ref{def:sufficiency}) is given in Appendix~\ref{app:vertex-separators} (Lemma~\ref{lem:separator-sufficiency}). In our construction (Section~\ref{sec:impl}), sufficiency holds by design. The architecture routes all inter-region communication exclusively through the interface states, with no bypass paths.

\begin{lemma}[Adjoint Factorization Through Interfaces]\label{lem:adjoint-factorization}
If $m_k$ is sufficient at boundary $k$, then for any loss $\mathcal{L}$ computed from variables in $R_{\geq k}$,
\begin{equation}\label{eq:adjoint-factor}
    \frac{\partial \mathcal{L}}{\partial \mathcal{R}_{<k}} = \left(\frac{\partial m_k}{\partial \mathcal{R}_{<k}}\right)^\top \frac{\partial \mathcal{L}}{\partial m_k}.
\end{equation}
That is, all reverse-mode adjoints from $\mathcal{R}_{\geq k}$ to $\mathcal{R}_{<k}$ factor exclusively through the interface $m_k$.
\end{lemma}
\begin{proof}
Since $m_k$ is sufficient, the loss can be written as $\mathcal{L} = \tilde{\mathcal{L}}(m_k; \theta_{\geq k})$ for some function $\tilde{\mathcal{L}}$ (Definition~\ref{def:sufficiency}). For any variable $w \in \mathcal{R}_{<k}$, the parameters $\theta_{\geq k}$ are independent of $w$, so $\mathcal{L}$ depends on $w$ only through $m_k$. By the chain rule, $\partial\mathcal{L}/\partial w = \left(\partial m_k/\partial w\right)^\top \partial\mathcal{L}/\partial m_k$. Since this holds for every $w \in \mathcal{R}_{<k}$, Eq.~\eqref{eq:adjoint-factor} follows.
\end{proof}

\begin{definition}[Bounded-Interface Model]\label{def:bounded-interface-model}
A region-partitioned computational graph is said to satisfy the \emph{bounded-interface property} if, for every boundary $k$, there exists a sufficient interface state $m_k$ with $\dim(m_k) \leq r_k$ for some prescribed interface budget $r_k$ independent of the total number of internal variables in the adjacent regions.
\end{definition}
\begin{remark}[Imposed vs. Discovered Interfaces]\label{rem:imposed-discovered-interfaces}
By Menger's theorem~\citep{bohme2001menger}, the minimum vertex separator between $\mathcal{R}_{<k}$ and $\mathcal{R}_{\geq k}$ in a typical neural network computational graph has size $d_k$ (the full hidden state) since each scalar component admits an independent path across the boundary under non-degeneracy (Appendix~\ref{app:vertex-separators}). The bounded-interface property is therefore an \emph{architectural imposition}, not a consequence of the model's existing structure. We choose to communicate $r_k \ll d_k$ scalars where exact sufficiency would formally require $d_k$. The architectural and empirical factors that make this compression viable are addressed in Sections~\ref{sec:impl} and~\ref{sec:experiments}.
\end{remark}
\begin{corollary}[Sequential Interface Representation]\label{cor:interface-chain}
For a bounded-interface model, all inter-region dependence factors through the interface chain $m_0 \to m_1 \to \cdots \to m_K$. In particular, inter-region forward computation is completely characterized by the transition maps,
\begin{equation}\label{eq:region-transition}
    m_{k+1} = R_k(m_k; \theta_k),
\end{equation}
where $R_k$ denotes the interface transition map associated with region $\mathcal{R}_k$, while all remaining computation is local to its corresponding region with local parameters $\theta_k$. By Lemma~\ref{lem:adjoint-factorization}, reverse-mode transport across regions must therefore also factor through this chain.
\end{corollary}
\begin{proof}
Fix $k \in \{0, \dotsc, K{-}1\}$. By sufficiency at boundary $k$ (Definition~\ref{def:sufficiency}), every node in $\mathcal{R}_{\geq k}$ (in particular every node internal to $\mathcal{R}_k$) depends on $\mathcal{R}_{<k}$ only through $m_k$. The outgoing interface state $m_{k+1}$ is computed from nodes in $\mathcal{R}_k$ together with the incoming interface $m_k$. Since all of these depend on $\mathcal{R}_{<k}$ only through $m_k$, the state $m_{k+1}$ is determined by $m_k$ and the local parameters $\theta_k$ alone, yielding the transition map~\eqref{eq:region-transition}. Applying this argument inductively for $k = 0, \dotsc, K{-}1$ produces the interface chain $m_0 \to m_1 \to \cdots \to m_K$. By Lemma~\ref{lem:adjoint-factorization}, all reverse-mode adjoints across region boundaries also factor through this chain.
\end{proof}

Corollary~\ref{cor:interface-chain} establishes that both forward and backward computation across regions reduce to operations on the interface chain $m_0 \to m_1 \to \cdots \to m_K$ rather than on the full hidden state. Since each interface transition $m_k \to m_{k+1}$ is a differentiable map between fixed-dimensional spaces, the inter-region adjoint recursion $\bar{m}_k = J_k^\top \bar{m}_{k+1}$ composes linear operators of size $r \times r$, an associative structure that admits a parallel scan formulation.

\subsection{Scan Parallel Backpropagation}\label{subsec:scan-parallel-backprop}

Backpropagation over bounded interface models follows directly from Lemma~\ref{lem:adjoint-factorization} and Corollary~\ref{cor:interface-chain}, admitting a parallel scan formulation and span advantage over standard sequential backpropagation. For simplicity in analysis and instantiation (Sections~\ref{subsec:scan-parallel-backprop}--\ref{sec:experiments}), we assume a uniform interface width $r_k = r$. The variable-width case is analogous. We refer to the propagation of adjoints across region boundaries via Jacobian transposes as \emph{adjoint transport}. Concretely, for a region transition $m_{k+1} = f_k(m_k)$, $\bar m_k = (\partial m_{k+1}/ \partial m_k)^\top \bar m_{k+1}$. The interface chain (Corollary~\ref{cor:interface-chain}) is seeded by an initial state $m_0$ and terminates at $m_K$. We assume that the loss $\mathcal{L}$ depends on all preceding regions exclusively through $m_K$, so that $\bar{m}_K = \partial\mathcal{L}/\partial m_K$ is well-defined and nonzero.

\begin{theorem}[Bounded-Interface Backpropagation]\label{thm:bi-backprop}
Let $\mathcal{G}$ be a region-partitioned computational graph satisfying the bounded-interface property, with region transitions as in Eq.~\eqref{eq:region-transition}, and let $\mathcal{L}$ be a scalar loss depending only on variables in the final region $\mathcal{R}_{K-1}$. 

Define the interface adjoint $\bar{m}_k := \partial \mathcal{L}/\partial m_k \in \mathbb{R}^{r}$. Then, for every $k=0, \dotsc, K-1$, 
$    \bar{m}_k = \left(\frac{\partial m_{k+1}}{\partial m_k}\right)^\top \bar{m}_{k+1},
$
and the local parameter gradient factorizes as $    \nabla_{\theta_k}\mathcal{L} = \left(\frac{\partial m_{k+1}}{\partial \theta_k}\right)^\top \bar{m}_{k+1}.
$
Moreover, writing $J_k := \partial m_{k+1} / \partial m_k \in \mathbb{R}^{r \times r}$, where the operator $J_k$ (the Jacobian) is the differential of the region transition map evaluated at the forward state $m_k$, the inter-region adjoint admits the exact product form,
\begin{equation}\label{eq:intr-reg-adj-prod}
    \bar{m}_k = \left(\prod_{j=k}^{K-1} J_j^\top \right) \bar{m}_K, \quad k=0, \dotsc, K-1,
\end{equation}
with the product ordered as $J^\top_k J^\top_{k+1}\cdots J^\top_{K-1}$, and the empty product ($k=K$) is the identity.
\end{theorem}
\begin{proof}
Fix $k \in \{0, \dotsc, K-1\}$. By Corollary~\ref{cor:interface-chain}, all influence of earlier regions on later computation factors through the interface chain, and in particular through $m_{k+1}$. Applying Lemma~\ref{lem:adjoint-factorization} at boundary $k{+}1$ gives $\bar{m}_k = (\partial m_{k+1}/\partial m_k)^\top \bar{m}_{k+1}$. A second application of the chain rule to the local parameters $\theta_k$, whose influence on $\mathcal{L}$ factors solely through $m_{k+1}$, gives $\nabla_{\theta_k}\mathcal{L} = (\partial m_{k+1}/\partial \theta_k)^\top \bar{m}_{k+1}$. Finally, unrolling the recursion of $\bar{m}_k$ yields $\bar{m}_k = J_k^\top \bar{m}_{k+1} = J_k^\top J_{k+1}^\top \bar{m}_{k+2} = \cdots = \left(\prod_{j=k}^{K-1} J_j^\top \right) \bar{m}_K$,
proving~\eqref{eq:intr-reg-adj-prod}. For $k = K$, the product is empty and therefore equals identity.
\end{proof}

The product form in Eq.~\eqref{eq:intr-reg-adj-prod} shows that inter-region adjoint transport reduces to the composition of linear operators $\{J_k^\top\}$. Since matrix multiplication is associative, this composition admits a parallel scan formulation.

\begin{proposition}[Scan Formulation of Backpropagation]\label{prop:scan-parallel-backprop}
The inter-region adjoints $\{\bar{m}_k\}_{k=0}^K$ admit an exact suffix scan formulation over the sequence $\{J_k^\top\}_{k=0}^{K-1}$. In particular, defining $P_k := \prod_{j=k}^{K-1} J_j^\top$, we have $\bar{m}_k = P_k \bar{m}_K$, where the products $\{P_k\}$ can be computed in parallel using any associative scan algorithm.
\end{proposition}
\begin{proof}
By Theorem~\ref{thm:bi-backprop}, the inter-region adjoints satisfy $\bar{m}_k = \left(\prod_{j=k}^{K-1} J_j^\top \right) \bar{m}_K = P_k \bar{m}_K$. Hence, computing all inter-region adjoints reduces to computing the suffix products $\{P_k\}_{k=0}^K$ of the sequence $\{J_k^\top\}_{k=0}^{K-1}$. Since matrix multiplication is associative, these suffix products can be computed by any standard parallel prefix/suffix scan algorithm~\citep{blelloch1990prefix}.
\end{proof}

This replaces the sequential dependency chain of reverse-mode backpropagation with a logarithmic-depth parallel scan over fixed-size operators.

\begin{corollary}[Region-Local Backward Independence]\label{cor:intra-reg-local-bwd}
Following Theorem~\ref{thm:bi-backprop}, once the interface adjoints $\{\bar{m}_k\}$ are available, all region-local backward computations are independent across regions.
\end{corollary}
\begin{proof}
By Theorem~\ref{thm:bi-backprop}, the contribution of region $k$ to the backward pass factors as $\nabla_{\theta_k}\mathcal{L}=\left(\partial m_{k+1}/\partial\theta_k\right)^\top\bar{m}_{k+1}$. The Jacobian $\partial m_{k+1}/\partial\theta_k$ depends only on the forward computation internal to region $k$, and therefore only on region-$k$ forward activations and local parameters $\theta_k$. It does not depend on activations, parameters, or adjoints from any other region, except through the incoming interface adjoint $\bar m_{k+1}$. Consequently, once the interface adjoints have been computed, each region-local backward computation depends only on local data and may be executed independently across regions.
\end{proof}

Proposition~\ref{prop:scan-parallel-backprop} and Corollary~\ref{cor:intra-reg-local-bwd} establish the structural decomposition. Crucially, this structure holds for any interface dimension, including $r_k = d_k$ (recovering BPPSA's full-rank scan~\citep{exact-para-scan-bp}). What separates the bounded-interface regime from the full-rank regime is not the existence of an associative decomposition, but the cost of each scan operation. We address this in Section~\ref{sec:impl}.

%% file: sections/04-implementation.tex
\section{Model Realization}\label{sec:impl}

\begin{table}[h!]
\centering
\begin{tabularx}{\linewidth}{c|c|c|c|c|c}
Method & FLOPs & Span & Operator & $I$ & $J_k$? \\
\hline
Sequential BP & $\Theta((BLD)^2)$ & $\Theta(K(BLD)^2)$ & $BLD \times BLD$ matvec & $\Theta(1)$ & No \\
Full-rank scan & $\Theta((BLD)^3)$ & $\Theta((BLD)^3 \log K)$ & $BLD \times BLD$ matmul & $\Theta(BLD)$ & Yes \\
LBI scan & $\Theta(r^3)$ & $\Theta(r^3 \log K)$ & $r \times r$ matmul & $\Theta(r)$ & Yes
\end{tabularx}
\caption{Comparison of inter-region gradient transport primitives. $I$ denotes arithmetic intensity (FLOPs/byte)~\citep{williams2009roofline}. Standard sequential BP applies Jacobian-vector products, while scan-based methods compose Jacobians via matrix-matrix products. The table compares per-step (sequential) or per-combine (scan) costs. The cost of materializing the Jacobian $J_k$ for scan-based methods is excluded and analyzed separately. Entries are derived in Appendix~\ref{app:representative-instances} (Remark~\ref{rem:table-transport-size-derivation}).}
\label{tab:transport-size-comparison}
\end{table}

Table~\ref{tab:transport-size-comparison} summarizes the cost comparison between the three transport regimes. We now address the concrete architecture that enforces bounded interfaces (Section~\ref{subsec:bi-architecture}), the computational cost of the resulting scan and Jacobian construction (Section~\ref{subsec:compute-cost}), and the backward algorithm (Section~\ref{subsec:backward-algorithm}).

\subsection{Bounded-Interface Architecture}\label{subsec:bi-architecture}

\paragraph{Instantiation Template.} We instantiate a bounded-interface model by associating each region $k$ with a local computation over activations $x_k$ and a low-dimensional interface state $m_k \in \mathbb{R}^r$. At each region $k$, the forward computation is defined as
\begin{equation}\label{eq:concrete-region-injection}
    x_k^\text{in} = x_\text{embed} + \operatorname{Dec}_k(m_k), \quad x_k^\text{out} = \Phi_k(x_k^\text{in}),
\end{equation}
where $\Phi_k$ denotes the region transformation map (e.g. attention, SSM, MLP), and $\operatorname{Dec}_k$ injects the interface state into the region. The interface state is then updated by
\begin{equation}\label{eq:concrete-region-arch}
    m_{k+1} = R_k(m_k) = \operatorname{LN}(m_k + \alpha_k \operatorname{Enc}_k(\operatorname{pool}(\Phi_k(x_\mathrm{embed} + \operatorname{Dec}_k(m_k))))),
\end{equation}
where $\operatorname{Enc}_k$ projects region outputs into the interface space, $\operatorname{pool}(\cdot)$ aggregates region outputs, $\alpha_k$ is a learned per-region scaling parameter, and $\operatorname{LN}$ denotes layer normalization~\citep{ba2016layer}. The initial interface state $m_0 \in \mathbb{R}^{r}$ is derived from the input tokens, and the terminal state $m_K$ serves as input to the terminal block / loss head. Thus, the construction in Eq.~\eqref{eq:concrete-region-arch} induces a region-partitioned computational graph in which all inter-region dependencies are mediated by the interface states $\{m_k\}_{k=0}^K$, thereby realizing the abstract region transition map $R_k(m_k; \theta_k)$ of Corollary~\ref{cor:interface-chain} in a concrete architectural form. Notice that when conditioned on $m_{k+1}$, no variable in region $\mathcal{R}_{\geq k+1}$ depends on variables in $\mathcal{R}_{\leq k}$, so the bounded-interface property holds.

\paragraph{Interface Design and Coupling.} A key design constraint in bounded-interface models is that all dynamic inter-region dependencies must be mediated through $m_k$, with no activation bypass paths from $x^\mathrm{out}_k$ to subsequent regions. Instead, we inject a shared static signal $x_\mathrm{embed}$ at each region (Eq.~\eqref{eq:concrete-region-injection}), providing every region with a common representational canvas. The interface state $m_k$ then needs to encode only the refinement from preceding regions rather than the full hidden representation, which explains why the interface rank can be small relative to the model dimension. The canvas does not violate the bounded-interface property (Definition~\ref{def:bounded-interface-model}): $x_\mathrm{embed}$ depends only on the input tokens, not on any region's computation, and therefore does not constitute an inter-region communication channel. Our implementation uses 2-layer MLPs for the encoder and decoder, which introduces $\Theta(D^2_k)$ parameters per encoder/decoder pair, and is comparable in order to the region transformation $\Phi_k$ itself (Table~\ref{tab:val-loss-final}). Alternative encoder/decoder realizations are discussed in Section~\ref{sec:disc}.

\subsection{Computational Cost}\label{subsec:compute-cost}

The cost of bounded-interface backpropagation decomposes into three components: Jacobian construction, scan composition, and region-local backward. We state the structural work--span~\citep{blelloch1996programming} result first, then derive each term.

\begin{proposition}[Work--Span Decomposition]\label{prop:work-span-decomposition}
The total work of bounded-interface backpropagation decomposes as,
\begin{equation}\label{eq:bi-total-work}
    W_\mathrm{bi} = \sum_{k=0}^{K-1} W_k^{J} + \Theta\left(Kr^3\right) + \sum_{k=0}^{K-1} W_k^\mathrm{local},
\end{equation}
and the total span satisfies
\begin{equation}\label{eq:bi-total-span}
    T_\mathrm{bi} = \max_k W_k^J + \Theta(r^3 \log K) + \max_k W_k^\mathrm{local}
\end{equation}
where the three terms correspond to (i) Jacobian construction across all regions (Eqs.~\eqref{eq:jacobian-work}--\eqref{eq:jacobian-arithmetic-intensity}), (ii) scan composition over interface Jacobians (Proposition~\ref{prop:scan-parallel-backprop}), and (iii) region-local backward computation (Corollary~\ref{cor:intra-reg-local-bwd}).
\end{proposition}
\begin{proof}
By Proposition~\ref{prop:scan-parallel-backprop}, inter-region adjoint transport reduces to a suffix scan over interface Jacobians $J_k$, contributing total work $\Theta(K r^3)$ and span $\Theta(r^3 \log K)$ for uniform $r$. Constructing each $J_k$ requires propagating $r$ perturbation directions through the linearized region $R_k$, at a cost we denote $W_k^J$ (quantified below). Since these constructions depend only on region-local forward caches and the incoming $m_k$, all $K$ constructions are independent, contributing total work $\sum_k W_k^J$ and span $\max_k W_k^J$. By Corollary~\ref{cor:intra-reg-local-bwd}, once the interface adjoints $\{\bar m_k\}$ are available, all region-local backward computations are independent, contributing total work $\sum_k W_k^\mathrm{local}$ and span $\max_k W_k^\mathrm{local}$. The three components execute sequentially: Jacobian construction requires the forward caches, the scan requires the constructed Jacobians, and region-local backward requires the scan output.
\end{proof}

\paragraph{Jacobian Construction.} Each $J_k := \partial m_{k+1}/\partial m_k \in \mathbb{R}^{r \times r}$ is constructed by propagating $r_k$ perturbation directions through the linearized region $R_k$ via automatic differentiation (AD)~\citep{baydin_AD}. Appendix~\ref{app:implement-details} details the construction procedure and its implementation. Let $F_k$ and $Q_k$ denote the total forward FLOPs and memory traffic for region $k$. Jacobian construction augments the forward computation with a basis dimension of size $r$, yielding total work
\begin{equation}\label{eq:jacobian-work}
    W_k^J = \mathrm{FLOPs}(J_k) = \Theta(r F_k).
\end{equation}
Under activation reuse across perturbation directions, memory traffic is nearly unchanged, giving arithmetic intensity~\citep{williams2009roofline},
\begin{equation}\label{eq:jacobian-arithmetic-intensity}
    I(J_k) = \Theta\left(r F_k/Q_k\right),
\end{equation}
an $r$-fold increase over the forward pass. The materialized Jacobian itself requires only $r^2$ scalars per region which is negligible relative to activation storage. Whether this places Jacobian construction in the compute-bound regime depends on the degree of activation reuse in practice. For $\sim$1B scale configurations ($r=64, d=BLD \approx 1.26\times 10^7$), the per-combine cost drops from $\Theta(d^3) \approx 2.0\times 10^{21}$ for the full-rank scan to $\Theta(r^3) \approx 2.6 \times 10^5$, a reduction of approximately $10^{16}\times$. Appendix~\ref{app:representative-instances} provides the full numerical analysis and shows compute-bound operation is achievable under sufficient reuse.

\paragraph{Scan and Region-Local Backward.} For the configurations in our experiments, the scan stage is negligible relative to Jacobian construction (Appendix~\ref{app:representative-instances}). Once interface adjoints are available (Corollary~\ref{cor:intra-reg-local-bwd}), each region's local backward costs $W_k^\mathrm{local} = \Theta(F_k)$, identical to standard reverse-mode seeded with $\bar m_{k+1}$. The dominant cost is Jacobian construction, with a work--depth tradeoff of an $(r{+}1)\times$ increase in total arithmetic in exchange for $O(\log K)$ scan depth and full region-level parallelism in both Jacobian construction and local backward computation.

\begin{algorithm}[h!]
\caption{Bounded-Interface Backpropagation}
\label{alg:bi-backprop}
\begin{algorithmic}[1]

\Require Region transition maps $\{R_k\}_{k=0}^{K-1}$, forward caches $\{C_k\}_{k=0}^{K-1}$, local parameters $\{\theta_k\}_{k=0}^{K-1}$, terminal interface adjoint $\bar m_K$
\Ensure Interface adjoints $\{\bar m_k\}_{k=0}^K$ and parameter gradients $\{\nabla_{\theta_k}\mathcal{L}\}_{k=0}^{K-1}$

\Statex $\rhd$ Phase 1: Jacobian construction

\ForAll{regions $k = 0, \dotsc, K-1$ \textbf{in parallel}}
    \State Construct interface Jacobians $J_k = \frac{\partial m_{k+1}}{\partial m_k} \in \mathbb{R}^{r \times r}$ from $C_k$ via AD
\EndFor
\Statex $\rhd$ Phase 2: Interface adjoint scan
\State Compute suffix products $\{P_k\}_{k=0}^K$ via parallel scan over $\{J_k^\top\}_{k=0}^{K-1}$
\ForAll{$k = 0, \dotsc, K$ \textbf{in parallel}}
    \State $\bar m_k \leftarrow P_k \bar m_K$
\EndFor

\Statex $\rhd$ Phase 3: Region-local backward
\ForAll{regions $k=0, \dotsc, K-1$ \textbf{in parallel}}
    \State Compute $\nabla_{\theta_k} \mathcal{L} = \left(\frac{\partial m_{k+1}}{\partial \theta_k} \right)^\top \bar m_{k+1}$ using $C_k$
    \State Compute any additional local adjoints for region $k$
\EndFor

\Return $\{\bar m_k\}_{k=0}^K$ and $\{\nabla_{\theta_k}\mathcal{L}\}_{k=0}^{K-1}$
\end{algorithmic}
\end{algorithm}

\subsection{Backward Algorithm}\label{subsec:backward-algorithm}

Algorithm~\ref{alg:bi-backprop} gives the backward procedure for a bounded-interface model, separating computation into three phases corresponding to the terms in Proposition~\ref{prop:work-span-decomposition}. Phase 1 constructs all interface Jacobians independently across regions. Phase 2 composes them via a parallel suffix scan (Proposition~\ref{prop:scan-parallel-backprop}) and is the sole sequential bottleneck with span $\Theta(r^3\log K)$. Phase 3 computes all region-local parameter gradients independently (Corollary~\ref{cor:intra-reg-local-bwd}). The three-phase structure exposes opportunities for overlapped execution, with a candidate streaming schedule given in Appendix~\ref{app:exec-and-schedule} (Algorithm~\ref{alg:bi-backprop-streaming}).

%% file: sections/05-experiments.tex
\section{Experiments}\label{sec:experiments}

\paragraph{Experimental Setup.} We evaluate bounded-interface backpropagation across four model backends: Mamba-2~\citep{mamba-2}, Mamba-3 SISO~\citep{mamba-3} (MIMO is excluded, see Appendix~\ref{app:experiment-details}), Transformer~\citep{vaswani2017attention}, and Hybrid (3$\times$ Mamba-3 SISO + 1$\times$ Transformer). All models are trained on 20.48M tokens from the FineWeb-Edu dataset~\citep{fineweb-edu} using the 32k LLaMA tokenizer~\citep{touvron2023llama} with standard next-token prediction at a 1024 context length. Our experiments evaluate bounded-interface backpropagation independently within each architecture, comparing dense and LBI variants at a fixed layer count, model dimension, and training budget, varying only the interface rank $r$. LBI models carry additional encoder/decoder parameters (10--18\%, Table~\ref{tab:val-loss-final}), but we hold intra-region parameters fixed rather than match total parameters. Reducing backend capacity would conflate weaker per-region computation with the interface bottleneck. Layer counts differ across backends (14 for Mamba-2/Mamba-3, 12 for Transformer/Hybrid) as they were chosen to yield comparable backend parameter counts (47-61M) rather than to match $K$. Each region comprises 2 consecutive layers, yielding $K{=}7$ regions for Mamba-2/Mamba-3 and $K{=}6$ for Transformer/Hybrid. Cross-architecture comparisons of the absolute LBI gap should therefore be interpreted as observational rather than controlled. Hyperparameters were tuned over dense baselines and transferred to LBI variants.

\begin{figure}[h!]
\centering
\resizebox{0.9\linewidth}{!}{%
\begin{tabular}{cc}
\begin{subfigure}{0.5\linewidth}
    \includegraphics[width=\linewidth]{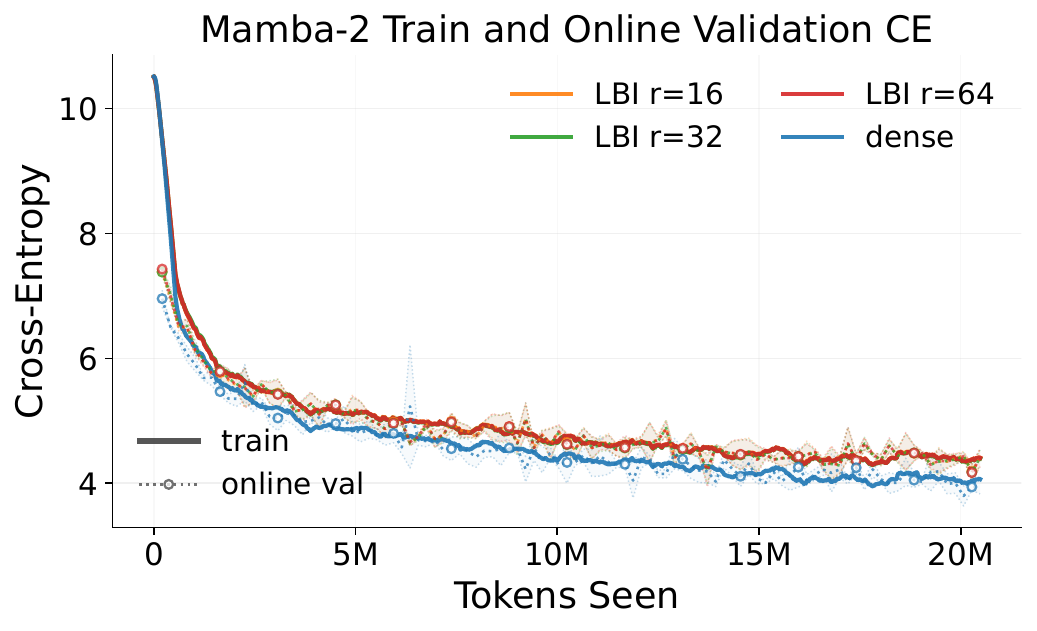}
    \caption{Mamba-2}
\end{subfigure} &
\begin{subfigure}{0.5\linewidth}
    \includegraphics[width=\linewidth]{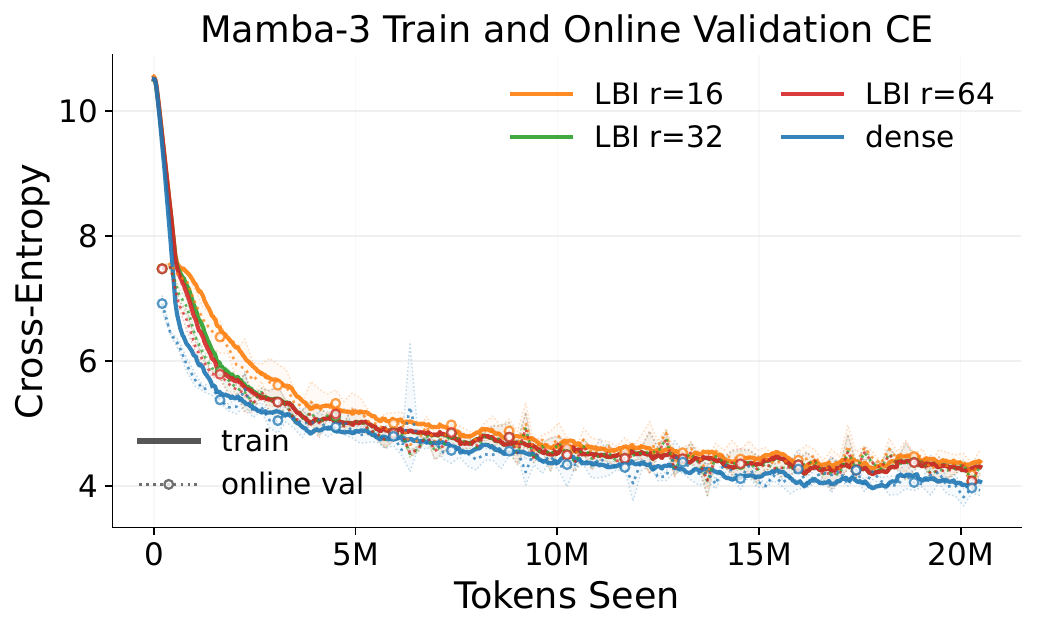}
    \caption{Mamba-3 SISO}
\end{subfigure} \\

\begin{subfigure}{0.5\linewidth}
    \includegraphics[width=\linewidth]{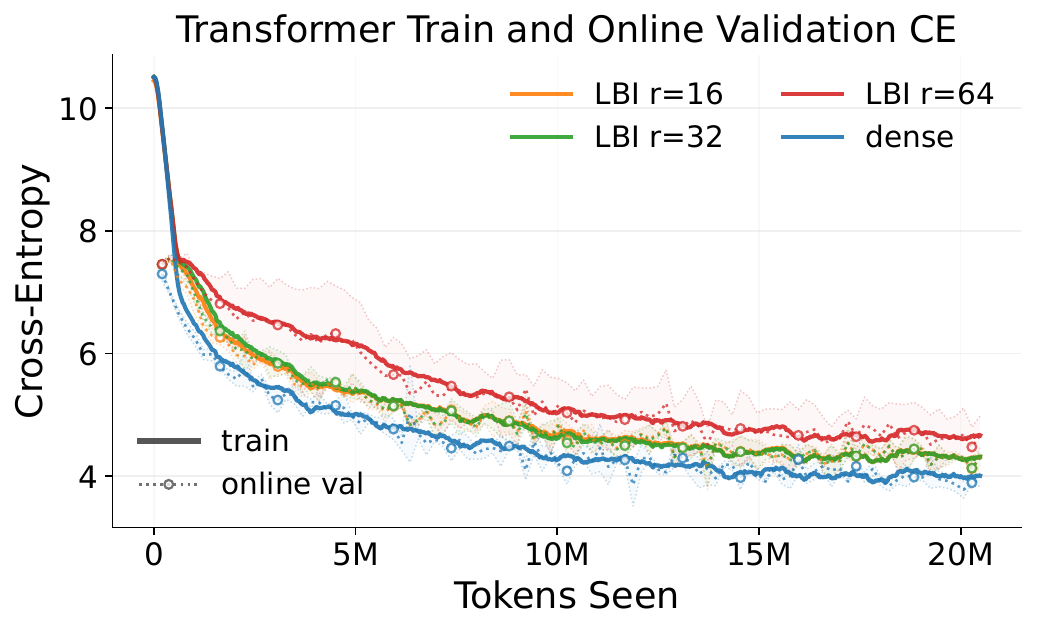}
    \caption{Transformer}
\end{subfigure} &
\begin{subfigure}{0.5\linewidth}
    \includegraphics[width=\linewidth]{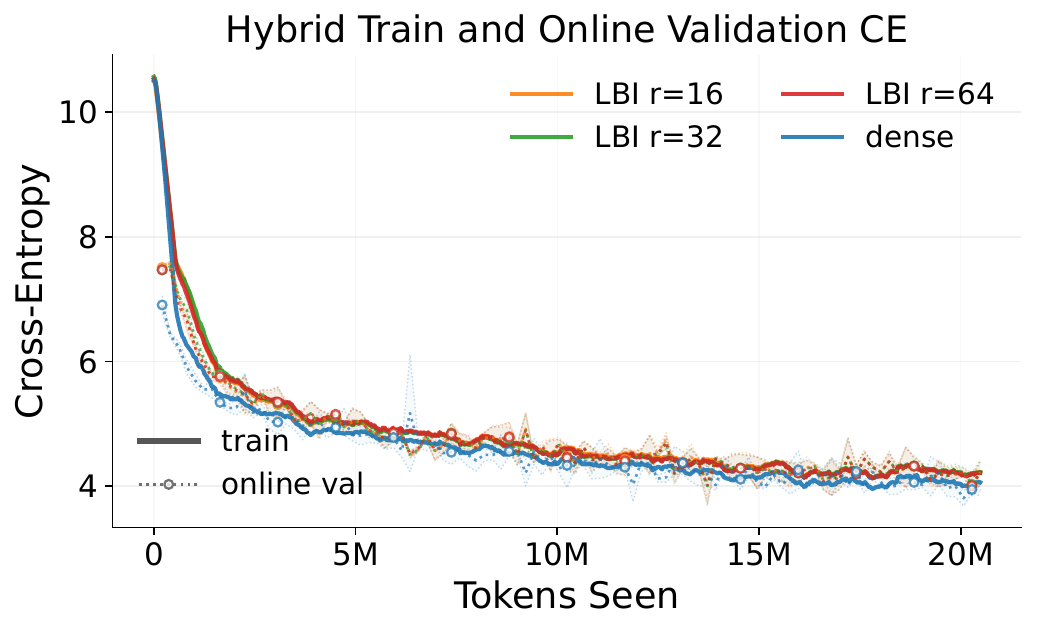}
    \caption{Hybrid}
\end{subfigure}
\end{tabular}
}
\caption{Smoothed training CE (solid lines) and periodic online validation CE (markers) on FineWeb-Edu across four model backends, averaged over three seeds. Training curves are smoothed with an exponential moving average for readability. Shaded regions indicate $\pm1$ standard deviation across seeds for validation loss. Online validation uses a small held-out budget per evaluation, final quantitative comparisons are reported using post-hoc evaluation on the full held-out sample (Table~\ref{tab:val-loss-final}).}
\label{fig:train-val-curves}
\end{figure}

\begin{table}[h!]
\centering
\small
\setlength{\tabcolsep}{4pt}
\resizebox{0.8\linewidth}{!}{%
\begin{tabular}{l|c|c|c|c|c}
Architecture & Rank $r$ & Backend $N$ & Interface $N$ & Total $N$ & Post-Hoc Val CE $\downarrow$ \\
\hline
\multirow{4}{*}{Mamba-2}
& -- & \multirow{4}{*}{51.33M} & -- & 75.91M & $4.061 \pm 0.010$ \\
& 16 &  & 9.05M & 84.95M & $4.413 \pm 0.016$ \\
& 32 &  & 9.23M & 85.14M & $4.404 \pm 0.009$ \\
& 64 &  & 9.60M & 85.51M & $4.408 \pm 0.019$ \\
\hline
\multirow{4}{*}{Mamba-3 SISO}
& -- & \multirow{4}{*}{53.52M} & -- & 78.09M & $4.070 \pm 0.009$ \\
& 16 &  & 9.05M & 87.14M & $4.393 \pm 0.108$ \\
& 32 &  & 9.23M & 87.33M & $4.310 \pm 0.004$ \\
& 64 &  & 9.60M & 87.70M & $4.312 \pm 0.003$ \\
\hline
\multirow{4}{*}{Transformer}
& -- & \multirow{4}{*}{47.25M} & -- & 63.63M & $3.992 \pm 0.005$ \\
& 16 &  & 3.52M & 67.16M & $4.318 \pm 0.031$ \\
& 32 &  & 3.63M & 67.26M & $4.314 \pm 0.051$ \\
& 64 &  & 3.84M & 67.48M & $4.646 \pm 0.504$ \\
\hline
\multirow{4}{*}{Hybrid}
& -- & \multirow{4}{*}{60.97M} & -- & 85.55M & $4.058 \pm 0.005$ \\
& 16 &  & 7.84M & 93.39M & $4.223 \pm 0.053$ \\
& 32 &  & 8.00M & 93.55M & $4.237 \pm 0.067$ \\
& 64 &  & 8.32M & 93.87M
 & $4.221 \pm 0.045$ \\
\end{tabular}
}
\caption{
    Final validation CE across architectures and interface ranks. Rank $r = \text{--}$ denotes the dense baseline (no interface), while finite $r$ corresponds to bounded-interface (LBI) models. Backend $N$ counts block parameters only and excludes token embeddings and output parameters; Interface $N$ counts additional bounded-interface parameters; Total $N$ counts all trainable parameters. All models reported are trained up to a fixed 20.48M tokens (Figure~\ref{fig:train-val-curves}) to enable controlled comparisons of relative behavior. Reported CE are post-hoc evaluations of final checkpoints, using the same fixed held-out FineWeb-Edu sample of 524,288 tokens, with mean $\pm$ standard deviation across three seeds.
}
\label{tab:val-loss-final}
\end{table}

\paragraph{Training Quality and Rank Sensitivity.} Bounded-interface backpropagation recovers gradients identical to standard reverse-mode AD applied to the same LBI model across all four architectures (Appendix~\ref{app:experiment-details}, Table~\ref{tab:gradient-parity}). Models train stably and remain competitive despite severe communication compression, converging within 0.16--0.35 CE of their dense counterparts (Table~\ref{tab:val-loss-final}). The modest gap is approximately constant across ranks $r \in \{16, 32, 64\}$ for most backends, suggesting that even the smallest rank tested ($r=16$) captures the dominant inter-region communication at this scale. The Transformer backend at $r{=}64$ exhibits high variance $(\pm0.504)$ from one unstable seed, which we attribute to initialization sensitivity. At $r\leq32$, training is stable (variance $\leq 0.051$). Since the CE gap persists despite the parameter advantage, the observed gap is a conservative upper bound on the bottleneck's cost. Appendix~\ref{app:region-size} provides a sensitivity analysis over region sizes. These results establish that bounded interfaces preserve model quality at the scales tested, validating the practical viability of the scan-based backward pass developed in Sections~\ref{sec:math}--\ref{sec:impl}.

%% file: sections/06-related-work.tex
\section{Related Work}\label{sec:related-work}
The formulation of backpropagation as a composition of Jacobian operators to enable parallel scan was introduced by BPPSA~\citep{exact-para-scan-bp}. However, BPPSA and related advances~\citep{deepPCR,waveScattering} rely on full-rank $d\times d$ Jacobians, making each scan combine intractable at modern hidden-state dimensions, and are further restricted to operations with closed-form Jacobians (e.g. convolutions), precluding application to Transformers and SSMs. Our work bridges this gap by observing that inter-layer communication can be restricted to compact latent interfaces, reducing operator Jacobian size to $r\times r$. Other efforts to break the sequential bottleneck include iterative methods~\citep{iterative-residual,lim2024parallelizingnonlinearsequentialmodels,scale-parallel-nRNN,layer-para-rnn}, decoupled approaches~\citep{de-info-reg,decoupled-neural-interfaces}, and Jacobian-Free Backpropagation~\citep{jfb}, which trade gradient exactness for throughput via approximations or auxiliary objectives. Pipeline parallel techniques~\citep{interlocking-backprop,chimera,GPipe} improve hardware utilization through micro-batch scheduling while leaving the $O(K)$ algorithmic depth intact. Our approach is complementary to these scheduling gains as it targets a reduction in the theoretical dependency depth from $O(K)$ to $O(\log K)$ without introducing gradient approximation.

%% file: sections/07-discussion.tex
\section{Discussion \& Limitations}\label{sec:disc}

\paragraph{From Algorithmic to Practical Parallelism.} Bounded-interface backpropagation decomposes into three phases: Jacobian construction, scan composition, and region-local backward. Our implementation provides a dedicated CUDA kernel for the suffix scan, but realizing the true wall-clock advantages (fused Jacobian materialization, overlapped scheduling, and fused region-local backward passes) requires further kernel engineering (Appendix~\ref{app:implement-details}). This decomposition has a direct implication for distributed training. Standard backpropagation creates an $O(K)$-deep chain of $BLD$-dimensional adjoint transmissions between devices. Although pipeline parallelism~\citep{interlocking-backprop,chimera,GPipe} improves utilization, it preserves this dependency. Bounded interfaces replace this chain with a single $O(\log K)$-depth scan over $K$ matrices of size $r \times r$. Because Jacobian construction and region-local backward are truly independent across regions with zero inter-region communication, this suffix scan becomes the only synchronous phase. In our experiments ($r = 64, K = 7$), the largest scan payload is approximately 56 KB in bf16, which is negligible relative to standard parameter gradient synchronization.

\paragraph{Interface Structure and Analysis.}
The bounded-interface formulation exposes inter-region Jacobians as low-dimensional objects that can be tracked throughout training, enabling direct study of gradient transport properties such as conditioning, spectral decay, and norm evolution, which are difficult to observe in standard backpropagation. Appendix~\ref{app:jacobian-structure} confirms that Mamba-3 SISO LBI chains remain contractive (suffix $\|P_k\|_2 < 1$) even when individual region transitions briefly exceed unit norms. This framework offers significant flexibility: it only mandates $r$-dimensional communication without prescribing the encoder/decoder architecture, thus preserving the parallelism guarantees of Section~\ref{sec:math} regardless of the encoder/decoder realization. While our implementation uses 2-layer MLPs (Section~\ref{subsec:bi-architecture}), this choice is not intrinsic to the framework. A linear projection would cost $\Theta(rD)$ with smaller constants, while structured alternatives (e.g., gated linear projections, cross-attention interfaces~\citep{jaegle2021perceiver}, or Monarch/butterfly matrices~\citep{dao2022monarch}) could reduce the overhead further or improve compression quality via task-specific inductive biases. Optimizing encoder/decoder designs to balance overhead and Jacobian stability remains an open design problem, similar to projection choices in MLA~\citep{deepseek-v2}.

\paragraph{Expressivity and Scale.} Any feedforward model imposes a Markov structure at every layer boundary, as layer $k$ is sufficient for computing all subsequent layers. The bounded-interface property makes this structure explicit and compresses it: the region $\mathcal{R}_{k}$ depends on all preceding regions only through the interface state $m_k \in \mathbb{R}^{r}$. The shared canvas $x_\mathrm{embed}$ (Section~\ref{subsec:bi-architecture}) ensures the interface encodes only inter-region refinements, explaining why small $r$ suffices. Appendix~\ref{app:jacobian-structure} provides spectral evidence that the dominant modes of inter-region communication concentrate in a low-dimensional subspace. Rank and region size sweeps (Table~\ref{tab:val-loss-final}, Appendix~\ref{app:region-size}) support this, as $r=16$ preserves stable training and maintains a near-constant CE gap up to $r=64$. While the current canvas uses cheap raw token embeddings, $r=16$ still achieves near-dense quality. Optimizing the information partitioning between the canvas and the interface remains a natural extension of the encoder/decoder design. At our scale (47-61M parameters, 20.48M tokens), we report CE loss on held-out samples as downstream benchmarks are uninformative at this budget. Finally, while relaxing the bounded-interface constraint to recover dense communication is possible, it falls outside our exact scan-parallel regime and is a natural extension for future work.

%% file: sections/08-conclusion.tex
\section{Conclusion}
\label{sec:conclusion}

We have introduced Latent Bounded Interfaces (LBI), an algorithmic framework that reduces the sequential depth of backpropagation from $O(K)$ to $O(\log K)$ by restricting inter-region communication to a low-dimensional latent space. Our core contribution reduces parallel-scan adjoint recursion from an intractable $O(d^3)$ operation to a tractable $O(r^3)$ cost, where $r \ll d$. Empirically, interfaces as small as $r=16$ maintain training quality, while the contractive nature of the resulting Jacobians ensures stable convergence across four architectures. Realizing the full wall-clock speedup of this approach requires further systems-level engineering to fuse Jacobian materialization and overlap the three phases of our algorithm. LBI reduces cross-device backward communication to an $O(\log K)$-depth scan over fixed-size $r \times r$ matrices, providing a mathematically rigorous foundation for region-parallel training and opening a new axis of parallelism for scaling deep learning beyond the sequential depth bottleneck.